\documentclass[11pt]{article}

\newif\ifnaaclsubmission
\naaclsubmissionfalse
\ifnaaclsubmission
  \usepackage[review]{acl}
\else
  \usepackage[final]{acl}
\fi

\usepackage{times}
\usepackage{latexsym}
\usepackage[T1]{fontenc}
\usepackage[utf8]{inputenc}
\usepackage{microtype}
\usepackage{inconsolata}
\usepackage{graphicx}
\usepackage{booktabs}
\usepackage{amsmath}
\usepackage{amssymb}
\usepackage{amsthm}
\usepackage{algorithm}
\usepackage{algpseudocode}
\usepackage{url}
\usepackage{cuted}
\usepackage{capt-of}

\newtheorem{proposition}{Proposition}
\theoremstyle{remark}
\newtheorem{remark}{Remark}

\title{LOCUS: Task-Aware Low-Rank Post-Training for Token-Efficient Language Generation}

\author{Dongfang Zhao (dzhao@uw.edu)}

\begin{document}
\maketitle

\begin{abstract}
Large language model serving costs scale directly with output sequence length, yet standard preference alignment often inflates response verbosity without improving utility.
We study whether the parameterization of post-training updates affects generation length: low-rank subspaces alter sequence length without modifying the alignment loss.
We present LOCUS, a method that selects a task-aware low-rank adaptation subspace to minimize output-token cost subject to a utility constraint.
Within this subspace, post-training retains the native preference objective with a frozen backbone.
Across Anthropic HH-RLHF dialogue preferences, we evaluate two $\sim$3B decoder backbones, Pythia-2.8B and Qwen2.5-3B, against protocol-matched full-parameter DPO and DrDPO branches and the released SamPO checkpoint. LOCUS reduces continuation length by up to 39.84\% on Pythia-2.8B and by 14.87--17.58\% on Qwen2.5-3B while updating only 0.24--0.28\% of model parameters, with no material change in the internal preference diagnostic.
\end{abstract}

\section{Introduction}
Serving large language models requires substantial computational and memory resources, with per-request latency and serving expenses governed by the number of autoregressively generated output tokens \citep{pope2023efficiently,kwon2023efficient}.
In production environments, each additional decoded token requires an iterative forward pass through the entire model depth, consuming memory bandwidth and expanding key-value cache storage \citep{leviathan2023fast,dao2022flashattention}.
Consequently, generating overly verbose responses directly diminishes operational throughput and escalates deployment costs.
Developing techniques that produce concise generations while maintaining strict answer quality represents a central objective for practical language model serving.

Post-training preference alignment methods align models with human values \citep{ouyang2022training,rafailov2023direct}.
Algorithms such as Direct Preference Optimization \citep{rafailov2023direct}, Distributionally Robust DPO \citep{wu2025drdpo}, and down-sampled divergence methods \citep{lu-etal-2024-eliminating} successfully steer model outputs toward preferred behaviors.
However, when optimized across all backbone parameters, these objectives frequently suffer from verbosity bias, wherein models learn that longer responses correlate with higher preference scores \citep{singhal2023long,saito2023verbosity,park-etal-2024-disentangling}.
These observations motivate examining whether the parameterization of preference updates affects generation length under an unchanged objective.

Prior attempts to curtail output length typically introduce ad-hoc regularizers or prompt modifications, yet these interventions exhibit clear deficiencies.
Methods that apply heuristic length penalties or negative length rewards risk distorting the underlying preference formulation, leading to premature termination or degraded response quality \citep{li-etal-2026-gr3}.
Similarly, prompting models for conciseness relies on fragile instruction following that degrades under distribution shifts \citep{zhou2023instruction}.
These approaches modify the learning objective or inference interface, leaving the effect of update parameterization on output length open to investigation.

Our insight toward this problem is that the parameterization of post-training updates can affect output-length dynamics: low-rank adaptation subspaces can alter sequence length without modifying the alignment objective.
Low-rank adaptation (LoRA)~\citep{hu2022lora} confines trainable updates to factorized matrices, making the update subspace a concrete variable for studying generation length.
By freezing the pretrained backbone and confining parameter updates to structured low-rank trajectories, the model lands on a materially different quality-token operating point under the exact original preference loss.

This paper turns the above insight into a post-training method, namely Length Optimization for Concise Utility-Preserving Sequences (LOCUS).
LOCUS selects a task-aware low-rank adaptation subspace on development data to minimize output-token cost subject to a utility constraint.
Within each candidate subspace, adaptation uses the native post-training objective with a frozen backbone.
LOCUS leaves the underlying loss formulation intact: it adds no length normalization beyond what a native objective already specifies, no explicit token penalty, and no brevity prompting.
With the objective fixed, LOCUS treats the low-rank subspace configuration (spanning adapter rank, module placement, layer scope, and training checkpoint step) as a structured design variable.

Figure~\ref{fig:opener_effect} summarizes the primary empirical results of LOCUS across three preference-optimization objective families on the Anthropic Helpful and Harmless (HH) benchmark \citep{bai2022training} using Pythia-2.8B \citep{biderman2023pythia}.
Across the three HH comparisons on Pythia-2.8B, LOCUS reduces continuation tokens by 20.73\% under DPO, 25.29\% under DrDPO, and 39.84\% under SamPO with negligible accuracy deltas ($\le 0.13$ pp), updating only 0.28\% of parameters.
Furthermore, controlled DPO and DrDPO comparisons on Qwen2.5-3B \citep{yang2024qwen25} show 14.87\% and 17.58\% token reductions, respectively, under 0.24\% trainable parameters.
\begin{figure}[t]
\centering
\includegraphics[width=\columnwidth]{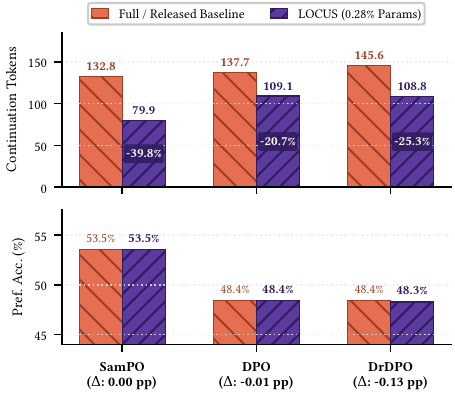}
\caption{Token and accuracy summary on Pythia-2.8B.}
\label{fig:opener_effect}
\end{figure}

In summary, this paper makes the following contributions:
\begin{itemize}
    \item We formulate token-efficient preference post-training as a task-aware low-rank subspace selection problem subject to an explicit utility constraint. The formulation preserves the native alignment objective; Section~\ref{sec:methodology} defines the problem and selection procedure.
    \item We analyze the trainable-parameter count and the exact inference-time equivalence between unmerged and merged low-rank updates. Section~\ref{sec:analysis} gives the propositions, proofs, and parameter accounting.
    \item We evaluate LOCUS on Anthropic HH-RLHF dialogue preferences with two $\sim$3B decoder backbones, Pythia-2.8B and Qwen2.5-3B, against protocol-matched full-parameter DPO and DrDPO branches and the released SamPO checkpoint. The protocol-matched branches yield 25.29\% and 17.58\% DrDPO reductions on Pythia-2.8B and Qwen2.5-3B, and continued SamPO adaptation yields 39.84\%, using a small trainable subspace; the complete protocol and results appear in Section~\ref{sec:evaluation}.
\end{itemize}

\section{Related Work}

\paragraph{Preference Optimization and Alignment.}
Preference alignment techniques train language models to align with human preferences using pairwise comparison data.
Reinforcement Learning from Human Feedback~\citep{christiano2017deep,stiennon2020learning,ouyang2022training} optimizes reward models using policy gradient techniques.
Direct Preference Optimization~\citep{rafailov2023direct} eliminates the separate reward modeling stage by deriving a closed-form substitution of the reward function into the policy objective.
Subsequent methods extend direct preference optimization through alternative objectives.
For example, IPO \citep{azar2024general} introduces regularized linear losses, whereas KTO \citep{ethayarajh2024kto} models prospect theory.
DrDPO \citep{wu2025drdpo} instead addresses reference distribution shifts.
Prior methods optimize objectives in the full parameter space or modify the loss formulation; LOCUS preserves the native preference loss while altering only the optimization subspace.

\paragraph{Length Bias and Output Conciseness.}
Language models aligned through preference optimization frequently develop severe length bias, generating unnecessarily long responses to maximize rewards \citep{singhal2023long,saito2023verbosity,dubois2024alpacafarm}.
Several studies investigate this phenomenon: LR-DPO~\citep{park-etal-2024-disentangling} disentangles length from quality through score adjustments, while SamPO~\citep{lu-etal-2024-eliminating} down-samples reference sequence length to diminish length reliance.
GR$^3$~\citep{li-etal-2026-gr3} applies group relative reward rescaling during reinforcement learning to alleviate length inflation.
Prior approaches introduce ad-hoc length regularization terms or heuristic prompting that distort the target preference distribution, whereas LOCUS achieves token efficiency through subspace parameterization without modifying the underlying objective.

\paragraph{Parameter-Efficient Fine-Tuning.}
Parameter-efficient fine-tuning (PEFT) adapts large language models by updating a small subset of parameters while freezing base weights.
Prominent methodologies include adapter insertion \citep{houlsby2019parameter}, prefix tuning \citep{li2021prefix}, prompt tuning \citep{lester2021power}, and low-rank adaptation \citep{hu2022lora}.
Recent works expand low-rank tuning through quantization \citep{dettmers2024qlora}, weight decomposition \citep{liu2024dora}, and adaptive rank allocation \citep{zhang2023adaptive}.
Theoretical investigations establish that language model adaptation possesses low intrinsic dimensionality, indicating that overparameterized updates are redundant for task learning \citep{aghajanyan2021intrinsic,sharma2023truth}.
Conventional parameter-efficient tuning treats low-rank adaptation solely as a resource-saving approximation to full fine-tuning; in contrast, LOCUS treats the adaptation subspace as a Pareto design variable to control generation length while preserving task utility.

\section{Methodology}\label{sec:methodology}

Figure~\ref{fig:architecture} summarizes the LOCUS procedure in three stages.
First, LOCUS attaches low-rank adapters to selected modules.
The pretrained backbone remains frozen, which confines parameter updates to those adapters.
The adapters are then trained with the original preference objective, allowing the experiment to test whether the update parameterization changes generation length without adding a length penalty or changing the loss.

Second, LOCUS trains a small set of candidate adapters that differ in rank, target modules, target layers, or training step.
It evaluates each candidate on development data and selects the shortest one that stays within the allowed utility difference from the baseline.
When a confirmation split is available, the selected adapter must pass that check before the held-out test evaluation.

Finally, LOCUS prepares the selected adapter for inference.
The adapter can be merged into the backbone for a single deployment, or kept separate when one backbone serves multiple task adapters.
The formal notation used in Figure~\ref{fig:architecture}, together with the native DPO and DrDPO objectives, is defined in Section~3.1.

\begin{figure}[t]
\centering
\includegraphics[width=\columnwidth]{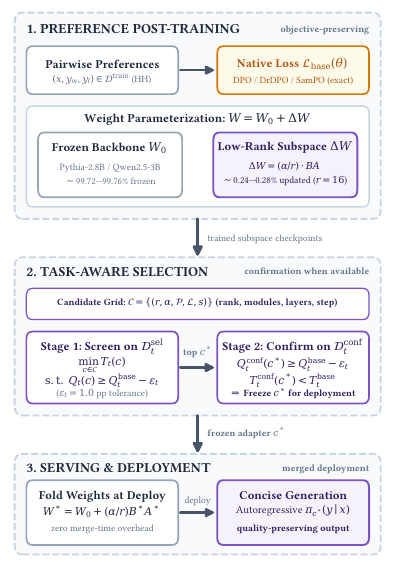}
\caption{LOCUS procedure from native-objective training to selected low-rank deployment.}
\label{fig:architecture}
\end{figure}

\begin{algorithm}[t]
\caption{LOCUS Subspace Selection}
\label{alg:locus_selection}
\footnotesize
\begin{algorithmic}[1]
\Require $\mathcal{L}_{\mathrm{base}}, W_0, \mathcal{D}_t^{\mathrm{train}}, \mathcal{D}_t^{\mathrm{sel}}, \mathcal{D}_t^{\mathrm{conf}}, \epsilon_t, \mathcal{C}$
\Ensure Selected adapter $c^*$ and policy $\pi_{c^*}$
\State $Q_t^{\mathrm{base}} \gets \mathrm{EvalUtility}(\pi_{\mathrm{base}}, \mathcal{D}_t^{\mathrm{sel}})$
\For{$c \in \mathcal{C}$}
    \State $\Delta W_c \gets \mathrm{Train}(\mathcal{L}_{\mathrm{base}}, W_0, \mathcal{D}_t^{\mathrm{train}}, c)$
    \State $(T_t(c), Q_t(c)) \gets \mathrm{Eval}(\pi_c, \mathcal{D}_t^{\mathrm{sel}})$
\EndFor
\State $\mathcal{C}_{\mathrm{feas}} \gets \{c \in \mathcal{C} \mid Q_t(c) \ge Q_t^{\mathrm{base}} - \epsilon_t\}$
\State $c^* \gets \arg\min_{c \in \mathcal{C}_{\mathrm{feas}}} T_t(c)$
\If{$\mathcal{D}_t^{\mathrm{conf}}$ exists}
\State $(T_{\mathrm{conf}}, Q_{\mathrm{conf}}) \gets \mathrm{Eval}(\pi_{c^*}, \mathcal{D}_t^{\mathrm{conf}})$
\State $Q_{\mathrm{conf}}^{\mathrm{base}} \gets \mathrm{EvalUtility}(\pi_{\mathrm{base}}, \mathcal{D}_t^{\mathrm{conf}})$
\State $T_{\mathrm{conf}}^{\mathrm{base}} \gets \mathrm{EvalLength}(\pi_{\mathrm{base}}, \mathcal{D}_t^{\mathrm{conf}})$
    \State $u \gets [Q_{\mathrm{conf}} \ge Q_{\mathrm{conf}}^{\mathrm{base}} - \epsilon_t]$
    \State $\ell \gets [T_{\mathrm{conf}} < T_{\mathrm{conf}}^{\mathrm{base}}]$
\If{$u \land \ell$}
    \State \Return Frozen adapter $c^*$ and model $\pi_{c^*}$
\Else
    \State \Return Fallback baseline $\pi_{\mathrm{base}}$
\EndIf
\Else
    \State \Return Unconfirmed adapter $c^*$
\EndIf
\end{algorithmic}
\end{algorithm}

\subsection{Objective-Preserving Adaptation}
We denote an input prompt by $x \in \mathcal{X}$ and an autoregressively generated continuation by $y = (y_1, y_2, \dots, y_T) \in \mathcal{Y}$.
We consider an autoregressive language model parameterized by frozen pretrained backbone weights $W_0 \in \mathbb{R}^{d_1 \times d_2}$.
Given a task dataset $\mathcal{D}$ with pairwise preference comparisons $(x, y_w, y_l)$, where $y_w$ denotes the preferred response and $y_l$ denotes the dispreferred response, standard alignment methods optimize a base objective $\mathcal{L}_{\mathrm{base}}(\theta)$ across all model parameters.
For example, under Direct Preference Optimization \citep{rafailov2023direct}, defining the implicit reward proxy $r_\theta(x, y) = \beta \log [\pi_\theta(y \mid x) / \pi_{\mathrm{ref}}(y \mid x)]$, the objective is:
\begin{equation}
\mathcal{L}_{\mathrm{DPO}}(\theta) = -\mathbb{E}_{\mathcal{D}} \left[ \log \sigma \left( r_\theta(x, y_w) - r_\theta(x, y_l) \right) \right],
\label{eq:dpo_loss}
\end{equation}
where $\beta > 0$ controls the divergence penalty from the reference policy $\pi_{\mathrm{ref}}$. The function $\sigma$ denotes the sigmoid.
Under Distributionally Robust DPO \citep{wu2025drdpo}, the objective robustifies against sample-level reference divergence:
\begin{equation}
\mathcal{L}_{\mathrm{DrDPO}}(\theta) = -\beta' \log \left( \frac{1}{B} \sum_{i=1}^B e^{-\ell_i(\theta)/\beta'} \right),
\label{eq:drdpo_loss}
\end{equation}
where $\ell_i(\theta)$ denotes the individual pairwise loss for sample $i$ in a microbatch of size $B$. The parameter $\beta'$ controls distributional robustness.

In LOCUS, we leave the native objective $\mathcal{L}_{\mathrm{base}}$ strictly intact.
We restrict trainable updates through the LoRA parameterization in Eq.~\eqref{eq:lora_param}:
\begin{equation}
W = W_0 + \Delta W = W_0 + \frac{\alpha}{r} B A,
\label{eq:lora_param}
\end{equation}
where $A \in \mathbb{R}^{r \times d_2} \sim \mathcal{N}(0, \sigma^2)$ and $B \in \mathbb{R}^{d_1 \times r}$ is zero-initialized. The adaptation uses rank $r \ll \min(d_1, d_2)$ with scaling factor $\alpha$.
We define an adaptation subspace configuration as a tuple $c = (r, \alpha, \mathcal{P}, \mathcal{L}, s) \in \mathcal{C}$. The sets $\mathcal{P} \subseteq \{\text{Q}, \text{K}, \text{V}, \text{O}, \text{MLP}\}$ and $\mathcal{L} \subseteq \{1, \dots, L\}$ specify the targeted weight modules and layer indices, respectively. For these targets, $s$ identifies the training optimization checkpoint step.
The resulting selection procedure is summarized in Algorithm~\ref{alg:locus_selection}. It first measures baseline utility on the task's selection split. Each candidate adapter is then trained with the same native objective and evaluated on that split for continuation length and preference utility. LOCUS discards candidates whose utility falls beyond the allowed tolerance, then selects the shortest remaining candidate. When a confirmation split is available, the selected candidate must also reduce length without exceeding the utility tolerance; otherwise, LOCUS falls back to the baseline. If no confirmation split is available, the selected adapter is returned as unconfirmed.

\subsection{Task-Aware Constrained Selection}
Different subspace configurations $c \in \mathcal{C}$ produce distinct behavioral trajectories during autoregressive decoding.
Define $T_t(c)$ as the mean continuation token length produced by configuration $c$ under greedy decoding on task $t$. The corresponding task utility is denoted by $Q_t(c)$.
For candidate selection, $Q_t(c)$ is evaluated using held-out chosen-versus-rejected sequence log-probability preference accuracy over $(x, y_w, y_l) \in \mathcal{D}_t$:
\begin{equation}
Q_t(c) = \frac{1}{|\mathcal{D}_t|} \sum_{(x, y_w, y_l)} \mathbb{I}(\pi_c(y_w \mid x) > \pi_c(y_l \mid x)).
\label{eq:utility_metric}
\end{equation}

LOCUS identifies the optimal adaptation configuration $c_t^*$ by solving a constrained minimization problem:
\begin{equation}
\begin{aligned}
c_t^* = \arg\min_{c \in \mathcal{C}} \; & T_t(c) \\
\text{subject to} \; & Q_t(c) \ge Q_t^{\mathrm{base}} - \epsilon_t,
\end{aligned}
\label{eq:locus_optimization}
\end{equation}
where $Q_t^{\mathrm{base}}$ is the benchmark utility achieved by the protocol-appropriate baseline on the same task. The tolerance $\epsilon_t \ge 0$ specifies the allowed utility decrease from that baseline (set to $1.0$ percentage points in this work).

To prevent split overfitting, LOCUS partitions development data into disjoint selection ($\mathcal{D}_t^{\mathrm{sel}}$) and confirmation ($\mathcal{D}_t^{\mathrm{conf}}$) splits. Candidate configurations in $\mathcal{C}$ are screened on $\mathcal{D}_t^{\mathrm{sel}}$ via Eq.~\eqref{eq:locus_optimization}; when available, the top candidate $c^*$ is verified on $\mathcal{D}_t^{\mathrm{conf}}$ before being frozen for held-out evaluation.

\subsection{Complexity and Serving Modalities}
The trainable parameter and optimizer-state footprint of LOCUS scales with the low-rank factors. For a targeted module, low-rank adaptation exposes $r(d_1+d_2)$ trainable entries; full-parameter adaptation exposes $d_1d_2$.
Consequently, gradient storage and optimizer states shrink with the trainable set. Under an Adam-style two-state optimizer, the primary adapter requires two state tensors over 7.86M entries for Pythia-2.8B or 7.37M for Qwen2.5-3B.
This accounting assumes two optimizer states per trainable entry; actual state storage depends on each experiment's optimizer. Because the frozen backbone also requires forward computation and activation backpropagation, total memory and training FLOPs require separate assessment.

During serving, LOCUS supports two deployment modalities. In single-tenant deployments, low-rank weights are pre-merged via $W^* = W_0 + \frac{\alpha}{r} B^* A^*$ prior to execution, incurring zero adapter-induced latency or parameter overhead. In multi-tenant deployments, a single frozen backbone dynamically serves multiple task adapters, though active adapters retain bookkeeping and memory overhead not benchmarked here.

\section{Analysis}\label{sec:analysis}

\subsection{Trainable Parameterization}\label{sec:analysis_parameterization}
Because LOCUS restricts every update to a low-rank subspace, the size of that restriction deserves a precise statement.
Full-parameter fine-tuning gives the optimizer one free parameter per entry of a weight matrix, whereas the low-rank parameterization gives it only the entries of the two factors: on a $2560 \times 2560$ module at $r=16$, 81,920 trainable entries rather than 6,553,600.
The proposition below states the count for a general module, with the remark that follows bounding what the count can be taken to mean.

\begin{proposition}[Trainable parameter-count reduction]\label{prop:adapter_count}
Let $W \in \mathbb{R}^{d_1 \times d_2}$ be a targeted linear module and let $r$ be a positive integer.
Full-parameter fine-tuning exposes $d_1d_2$ trainable scalar entries, whereas the LoRA parameterization
$W = W_0 + (\alpha/r)BA$, with fixed $W_0$, fixed $\alpha \ne 0$, $B \in \mathbb{R}^{d_1 \times r}$, and $A \in \mathbb{R}^{r \times d_2}$, exposes $r(d_1+d_2)$ trainable factor entries.
Thus the factor parameterization has fewer exposed trainable entries exactly when $r(d_1+d_2) < d_1d_2$.
\end{proposition}
\begin{proof}
The complete proof, including the distinction between exposed factor entries and the intrinsic dimension of the represented update set, is given in Appendix~\ref{sec:appendix_proof}.
\end{proof}
\begin{remark}[Interpretation of the count]
The proposition is an implementation-level parameter-count statement.
It does not claim that the low-rank factors select particular semantic or verbosity-specific directions, nor that the count alone determines generation length, preference accuracy, total memory, or training speed.
\end{remark}

For the selected Pythia-2.8B attention configuration ($L=32, d=2560, r=16$), the combined adapter accounts for $N_{\mathrm{attn}} = Lr((3d+d)+(d+d)) = 7{,}864{,}320$ trainable parameters ($0.2826\%$ of the model). Similarly, for Qwen2.5-3B ($L=36, r=16$ on \texttt{q\_proj}, \texttt{k\_proj}, \texttt{v\_proj}, \texttt{o\_proj}), the adapter comprises $7{,}372{,}800$ trainable parameters ($0.2383\%$ of 3.09B). These calculations explain the parameter counts in Table~\ref{tab:backbones}. Translating those counts into end-to-end memory or latency requires measurements of the execution path.

\subsection{Inference-Time Merge Equivalence}\label{sec:analysis_merge}
A second property is required for the deployment discussion: after training, LOCUS may either keep an adapter branch or merge its update into the corresponding frozen linear module.
The following proposition identifies the exact condition under which these two inference representations compute the same function.

\begin{proposition}[Exact equivalence after adapter merging]\label{prop:merge_equivalence}
Consider a linear module with input $h$, frozen weight $W_0$, and a LoRA update $(\alpha/r)BA$.
At inference time, with adapter dropout disabled, the unmerged computation and the computation using the merged weight
$W^* = W_0 + (\alpha/r)BA$ produce exactly the same module output for every input $h$.
\end{proposition}
\begin{proof}
The complete algebraic proof and its assumptions are given in Appendix~\ref{sec:appendix_proof}.
\end{proof}
\begin{remark}[Scope of merge equivalence]
The equality concerns the output of a linear module in inference mode. Deployments that produce this same output can have different latency and memory costs because an unmerged adapter adds computation and bookkeeping. Furthermore, finite-precision or quantized implementations may introduce small numerical differences.
\end{remark}

\section{Evaluation}\label{sec:evaluation}

\subsection{Experimental Setup}\label{sec:setup}

\paragraph{Backbones.}
We evaluate LOCUS on two decoder-only backbone families at roughly 3B parameter scale: Pythia-2.8B~\citep{biderman2023pythia}, based on GPT-NeoX~\citep{black2022gpt}, and Qwen2.5-3B~\citep{yang2024qwen25}, a modern Llama-style architecture with grouped-query attention, SwiGLU~\citep{shazeer2020glu}, and RMSNorm~\citep{zhang2019rmsnorm}. Appendix Table~\ref{tab:backbones} details the architectural differences and target modules for each backbone.

\paragraph{Baselines.}
We compare LOCUS with a protocol-matched full-parameter branch while fixing the training pool, starting checkpoint, native objective, and data splits. Controlled DPO~\citep{rafailov2023direct} and DrDPO~\citep{wu2025drdpo} first train a shared full-parameter supervised fine-tuned (SFT) checkpoint, then compare full-parameter preference optimization with low-rank LOCUS from that checkpoint. SamPO~\citep{lu-etal-2024-eliminating} instead continues from the official released Pythia-2.8B HH-RLHF Iterative SamPO checkpoint, while Qwen2.5-3B repeats the shared-SFT comparison under both native objectives. Appendix Table~\ref{tab:baselines} makes every starting point and branch parameterization explicit.

\paragraph{Datasets.}
Primary evaluation uses Anthropic Helpful and Harmless (HH-RLHF)~\citep{bai2022training}, an open-ended multi-turn dialogue preference benchmark. DPO and DrDPO use 8,552 frozen test pairs, while SamPO uses its established 256-example split. The additional datasets are Anthropic HH-RLHF \texttt{harmless-base} for safety refusal and Orca DPO~\citep{mukherjee2023orca} for single-turn instruction following. Each uses a separate 256-example development evaluation; neither is presented as an external test result. Appendix Table~\ref{tab:datasets} records these roles.

\paragraph{Platform.}
All training and inference runs use one NVIDIA A100-PCIE-80GB GPU with bfloat16 mixed precision. Activation checkpointing is enabled for the DPO and DrDPO full-parameter runs with 512-token sequences, whereas the SamPO comparison follows the published recipe. The approximately 3B scale makes the full-vs-low-rank protocol reproducible on this hardware.

\paragraph{Metrics.}
We measure continuation length and held-out preference accuracy. Generation uses greedy decoding ($\text{do\_sample}=\text{False}$, $\text{max\_new\_tokens}=256$). We count continuation tokens exclusive of the end-of-sequence (EOS) token and record 256-token limit hits. Preference accuracy compares chosen and rejected sequence log-probabilities under the evaluated policy (Eq.~\eqref{eq:utility_metric}). We choose this metric because fixed preference labels and an explicit scoring rule enable reproducible model comparisons.

\subsection{Baseline Comparisons}
\begin{figure}[t]
\centering
\includegraphics[width=\columnwidth]{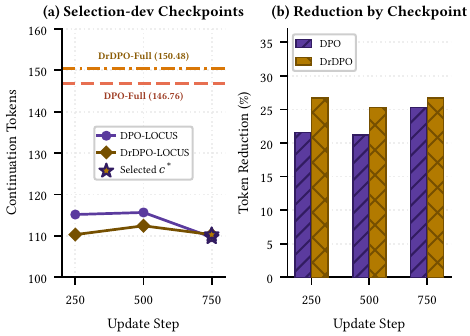}
\caption{Constrained checkpoint selection.}
\label{fig:dynamics}
\end{figure}

Figure~\ref{fig:dynamics} reports the measured 250-, 500-, and 750-step selection-dev checkpoints for DPO and DrDPO, with full-parameter baselines at 146.76 and 150.48 tokens. For DPO on Pythia-2.8B, both branches train on 160,544 Anthropic HH pairs with $\beta=0.1$ from the same SFT checkpoint; DPO-FULL updates all weights, whereas DPO-LOCUS uses a low-rank adapter. On the 8,552 frozen test pairs, LOCUS reduces mean continuation length from 137.67 to 109.12 tokens (20.73\%) and changes preference accuracy from 48.40\% to 48.39\% ($-0.01$ pp), using 7,864,320 trainable parameters (0.2826\%). Despite nonmonotonic trajectories, utility-constrained selection chooses step 750 in both comparisons.

\begin{figure}[t]
\centering
\includegraphics[width=\columnwidth]{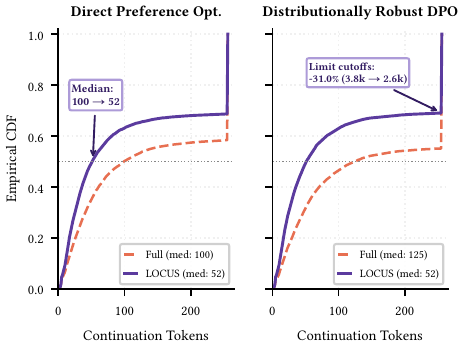}
\caption{Continuation length CDF.}
\label{fig:length_cdf}
\end{figure}

Figure~\ref{fig:length_cdf} presents the continuation-length CDF across the 8,552 frozen test prompts. Under Distributionally Robust DPO (DrDPO) on Pythia-2.8B \citep{wu2025drdpo}, DRDPO-FULL updates all weights while DRDPO-LOCUS updates only its adapter from the same SFT checkpoint, using $\beta=0.1$ and $\beta'=1.0$. LOCUS reduces mean length from 145.61 to 108.79 tokens (25.29\%) and changes preference accuracy from 48.39\% to 48.26\% ($-0.13$ pp); the selected step-750 candidate reaches 26.70\% reduction on selection-dev. The median falls from 100 to 52 tokens under DPO and from 125 to 52 under DrDPO, while limit hits fall by 24.70\% and 30.98\%.

For SamPO, LOCUS starts from the official Pythia-2.8B Anthropic HH Iterative SamPO checkpoint \citep{lu-etal-2024-eliminating}, which incorporates per-token length normalization into preference optimization. On the 256-example HH evaluation split, the released checkpoint produces a mean continuation of 132.77 tokens with 53.52\% preference accuracy. We train a low-rank adapter with $r=16$ and $\alpha=32$ on attention modules under the native SamPO objective ($\beta=0.05$). As summarized in Figure~\ref{fig:opener_effect} and Appendix Table~\ref{tab:main_results}, LOCUS reduces mean continuation length from 132.77 to 79.88 tokens (52.89 tokens; 39.84\%) while preserving 53.52\% preference accuracy ($\Delta=0.00$ pp).

\subsection{Rank Sensitivity}
Using the Pythia-2.8B backbone, we measure sensitivity to adapter rank $r \in \{4, 8, 16, 32\}$ on a 64-example Anthropic HH screening subset under attention adaptation with the SamPO objective ($\alpha / r = 2$). Rank controls the dimensionality of the trainable update, so we keep the target modules, objective, data split, and scaling rule fixed while varying only $r$. This screen uses a separate subset from the candidate pool $\mathcal{C}$ evaluated in Section~\ref{sec:evaluation}, so it characterizes rank sensitivity without redetermining that selection. Figure~\ref{fig:rank_sensitivity} reports the resulting token changes and internal preference diagnostics.

\begin{figure}[t]
\centering
\includegraphics[width=\columnwidth]{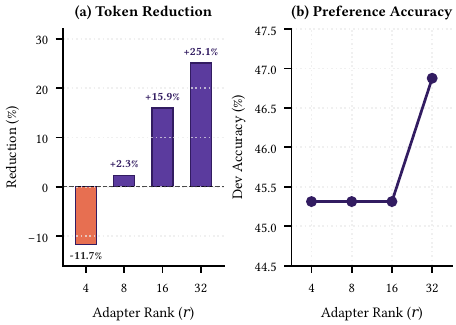}
\caption{Sensitivity to adapter rank.}
\label{fig:rank_sensitivity}
\end{figure}

Because the objective, attention placement, data split, and step budget are fixed, this screen isolates the effect of rank within the candidate adapter subspace.
At rank $r=4$, the model exhibits length inflation (+11.74\% tokens), whereas $r=8$ produces a small reduction of 2.32\%.
The reduction increases to 15.95\% at $r=16$ and 25.10\% at $r=32$ (Figure~\ref{fig:rank_sensitivity}(a)), so the measured length effect varies materially across the tested ranks.
The internal preference diagnostic remains 45.31\% for $r \in \{4,8,16\}$ and rises to 46.88\% at $r=32$ (Figure~\ref{fig:rank_sensitivity}(b)); within this screen, token reduction therefore does not follow a monotonic change in that diagnostic.

\subsection{Target-Module Ablation}
Using Pythia-2.8B, we perform a controlled ablation over which existing linear projections receive LoRA updates on a 64-example Anthropic HH screening subset with $r=16$, keeping the backbone and layer coverage fixed.
The five candidate target sets are query-key-value projections only (QKV), attention output projections only (O), their combination (Attention-All), feed-forward network projections only (MLP), and all linear projections (All-Linear).

\begin{figure}[t]
\centering
\includegraphics[width=\columnwidth]{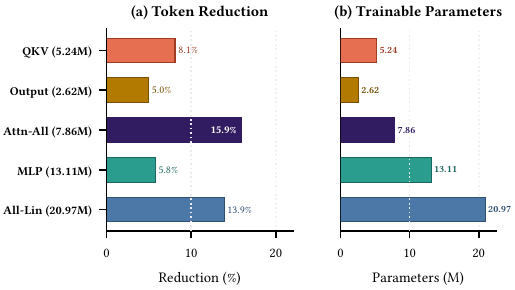}
\caption{Target-module ablation.}
\label{fig:placement_ablation}
\end{figure}

Figure~\ref{fig:placement_ablation} compares token reduction with trainable parameter count across the five target sets. QKV and output-only adaptation yield 8.09\% and 4.96\% reductions, while Attention-All reaches 15.95\%. MLP adaptation yields 5.75\%; All-Linear reaches 13.88\% with 2.7$\times$ more trainable parameters than Attention-All (20.97M vs.\ 7.86M). The internal preference diagnostic remains 45.31\% across all placements.

\subsection{Cross-Task Generality}
We test whether the measured token reduction extends beyond the HH dialogue setting when the preference data and task semantics change.
Using Pythia-2.8B, we evaluate the HH comparison alongside Anthropic HH-RLHF \texttt{harmless-base} and the Orca DPO instruction-following dataset (Figure~\ref{fig:crosstask}). These evaluations cover dialogue, safety, and instruction-following settings, each with a task-specific adapter; the HH row reports the 8,552-pair frozen-test result, Harmless and Orca DPO each use a 256-example development split, and the preference numbers remain internal chosen-vs-rejected diagnostics rather than external quality evaluations.

\begin{figure}[t]
\centering
\includegraphics[width=\columnwidth]{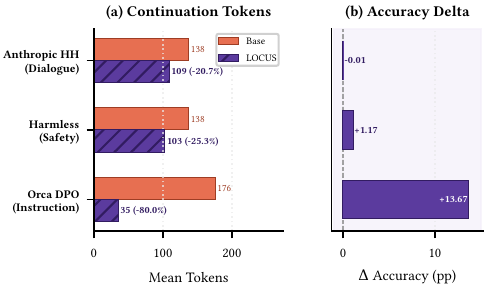}
\caption{Cross-task evaluations.}
\label{fig:crosstask}
\end{figure}

Figure~\ref{fig:crosstask}(a) reports task-specific length changes. The HH comparison reduces mean continuation length from 137.67 to 109.12 tokens (20.73\%). The Harmless evaluation reduces it from 137.53 to 102.75 tokens (25.29\%), while the Orca DPO evaluation reduces it from 175.95 to 35.25 tokens (79.97\%).
Figure~\ref{fig:crosstask}(b) reports the corresponding internal preference diagnostics: HH changes from 48.40\% to 48.39\% ($-0.01$ pp), Harmless changes from 57.03\% to 58.20\% ($+1.17$ pp), and Orca DPO changes from 69.14\% to 82.81\% ($+13.67$ pp).

\subsection{Cross-Backbone Transfer to Qwen2.5-3B}\label{sec:qwen_validation}
We test whether the token-reduction effect transfers across architectural families on Qwen2.5-3B \citep{yang2024qwen25}, a distinct decoder backbone family at comparable $\sim$3B scale.
Following the identical experimental discipline, both branches begin from a single shared full-parameter SFT checkpoint trained for one epoch on the 160,288 Anthropic HH training pairs.
Candidate low-rank adapters ($r=16, \alpha=32$ on the attention projections of all 36 layers; 7.37M parameters, 0.2383\%) are screened on the 256-pair selection split and confirmed on a disjoint 256-pair split before frozen held-out evaluation.

\begin{figure}[t]
\centering
\includegraphics[width=\columnwidth]{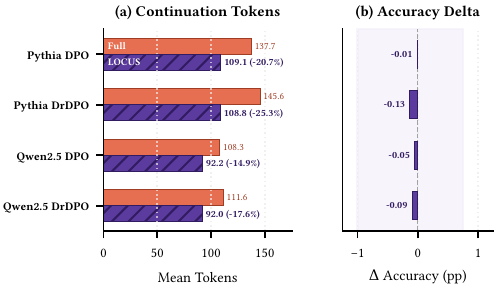}
\caption{Cross-backbone token and accuracy results.}
\label{fig:cross_backbone}
\end{figure}

Figure~\ref{fig:cross_backbone} reports controlled frozen test results on the 8,552 Anthropic HH test pairs for both DPO and DrDPO alongside the established Pythia-2.8B baselines.
Under DPO ($\beta=0.1$), LOCUS (step 250) reduces mean continuation length from 108.26 to 92.16 tokens (14.87\% reduction; median 75 to 51 tokens), while held-out preference accuracy changes from 48.69\% to 48.64\% ($-0.05$ pp delta).
Under DrDPO ($\beta=0.1, \beta'=1.0$), LOCUS (step 500) reduces mean continuation length from 111.58 to 91.97 tokens (17.58\% reduction; median 81 to 52 tokens), with preference accuracy changing from 48.64\% to 48.55\% ($-0.09$ pp delta).
Both comparisons meet the strong positive criterion of at least $2\%$ token reduction with a utility delta of at least $-1.0$ pp.

\section{Conclusion}
This paper presents LOCUS, which freezes the pretrained backbone and selects a low-rank adaptation subspace while preserving the native preference objective. On Anthropic HH-RLHF, we compare LOCUS with protocol-matched full-parameter DPO and DrDPO branches and with the released SamPO checkpoint on Pythia-2.8B, then repeat the comparisons on Qwen2.5-3B. Each comparison matches the objective, checkpoint, data split, and decoding protocol, measuring continuation length and internal chosen-vs-rejected preference accuracy. The protocol-matched DrDPO branch yields 25.29\% and 17.58\% reductions on Pythia-2.8B and Qwen2.5-3B, and continued SamPO adaptation yields 39.84\%, all while updating fewer than 0.3\% of model parameters. Additional development evaluations extend the measured reductions to safety and instruction-following data under the same internal preference diagnostic.

\clearpage

\section*{Limitations}
LOCUS selects a low-rank adaptation subspace for post-training, with the following limitations.
The evaluation covers two decoder-only backbone families at roughly 3B scale, Pythia-2.8B and Qwen2.5-3B, under greedy decoding. Larger models and sampling-based inference remain untested.

The candidate pool $\mathcal{C}$ uses four discrete ranks and coarse projection groupings. Continuous rank allocation, finer layer selection, and gradient-informed pruning remain outside the study.

The cross-task evaluations broaden coverage across dialogue, safety, and instruction following, but do not predict the reduction available on an unseen task.

\bibliography{references}

@inproceedings{rafailov2023direct,
  title = "Direct Preference Optimization: Your Language Model is Secretly a Reward Model",
  author = "Rafailov, Rafael and Sharma, Archit and Mitchell, Eric and Manning, Christopher D. and Ermon, Stefano and Finn, Chelsea",
  booktitle = "Advances in Neural Information Processing Systems",
  volume = "36",
  pages = "53728--53741",
  year = "2023",
  url = "https://proceedings.neurips.cc/paper_files/paper/2023/hash/a85b405ed65c6477a4fe8302b5e06ce7-Abstract-Conference.html"
}

@inproceedings{wu2025drdpo,
  title = "Towards Robust Alignment of Language Models: Distributionally Robustifying Direct Preference Optimization",
  author = "Wu, Junkang and Xie, Yuexiang and Yang, Zhengyi and Wu, Jiancan and Chen, Jiawei and Gao, Jinyang and Ding, Bolin and Wang, Xiang and He, Xiangnan",
  booktitle = "Proceedings of the International Conference on Learning Representations",
  year = "2025",
  url = "https://openreview.net/forum?id=CbfsKHiWEn"
}

@inproceedings{lu-etal-2024-eliminating,
  title = "Eliminating Biased Length Reliance of Direct Preference Optimization via Down-Sampled {KL} Divergence",
  author = "Lu, Junru and Li, Jiazheng and An, Siyu and Zhao, Meng and He, Yulan and Yin, Di and Sun, Xing",
  booktitle = "Proceedings of the 2024 Conference on Empirical Methods in Natural Language Processing",
  month = nov,
  year = "2024",
  address = "Miami, Florida, USA",
  publisher = "Association for Computational Linguistics",
  pages = "1047--1067",
  url = "https://aclanthology.org/2024.emnlp-main.60/"
}

@inproceedings{park-etal-2024-disentangling,
  title = "Disentangling Length from Quality in Direct Preference Optimization",
  author = "Park, Ryan and Rafailov, Rafael and Ermon, Stefano and Finn, Chelsea",
  booktitle = "Findings of the Association for Computational Linguistics: ACL 2024",
  month = aug,
  year = "2024",
  address = "Bangkok, Thailand",
  publisher = "Association for Computational Linguistics",
  pages = "4998--5017",
  url = "https://aclanthology.org/2024.findings-acl.297/"
}

@article{li-etal-2026-gr3,
  title = "Tackling Length Inflation Without Trade-offs: Group Relative Reward Rescaling for Reinforcement Learning",
  author = "Li, Zichao and Lou, Jie and Dong, Fangchen and Fan, Zhiyuan and Ren, Mengjie and Lin, Hongyu and Han, Xianpei and Zhang, Debing and Sun, Le and Lu, Yaojie and Yu, Xing",
  journal = "arXiv preprint arXiv:2603.10535",
  year = "2026",
  url = "https://arxiv.org/abs/2603.10535"
}

@inproceedings{azar2024general,
  title = "A General Theoretical Paradigm to Understand Learning from Human Preferences",
  author = "Azar, Mohammad Gheshlaghi and Guo, Zhaohan Daniel and Piot, Bilal and Munos, Remi and Rowland, Mark and Valko, Michal and Calandriello, Daniele",
  booktitle = "Proceedings of the International Conference on Artificial Intelligence and Statistics",
  volume = "238",
  pages = "4447--4455",
  year = "2024",
  url = "https://proceedings.mlr.press/v238/gheshlaghi-azar24a.html"
}

@inproceedings{ethayarajh2024kto,
  title = "Model Alignment as Prospect Theoretic Optimization",
  author = "Ethayarajh, Kawin and Xu, Winnie and Muennighoff, Niklas and Jurafsky, Dan and Kiela, Douwe",
  booktitle = "Proceedings of the 41st International Conference on Machine Learning",
  pages = "12634--12651",
  year = "2024",
  url = "https://proceedings.mlr.press/v235/ethayarajh24a.html"
}

@inproceedings{ouyang2022training,
  title = "Training language models to follow instructions with human feedback",
  author = "Ouyang, Long and Wu, Jeffrey and Jiang, Xu and Almeida, Diogo and Wainwright, Carroll and Mishkin, Pamela and Zhang, Chong and Agarwal, Sandhini and Slama, Katarina and Ray, Alex and Schulman, John and Hilton, Jacob and Kelton, Fraser and Miller, Luke and Simens, Maddie and Askell, Amanda and Welinder, Peter and Christiano, Paul F. and Leike, Jan and Lowe, Ryan",
  booktitle = "Advances in Neural Information Processing Systems",
  volume = "35",
  pages = "27730--27744",
  year = "2022",
  url = "https://proceedings.neurips.cc/paper_files/paper/2022/hash/b1efde53be364a73914f58805a001731-Abstract-Conference.html"
}

@inproceedings{stiennon2020learning,
  title = "Learning to summarize with human feedback",
  author = "Stiennon, Nisan and Ouyang, Long and Wu, Jeffrey and Ziegler, Daniel and Lowe, Ryan and Voss, Chelsea and Radford, Alec and Amodei, Dario and Christiano, Paul F.",
  booktitle = "Advances in Neural Information Processing Systems",
  volume = "33",
  pages = "3008--3021",
  year = "2020",
  url = "https://proceedings.neurips.cc/paper/2020/hash/1f89885d556929e98d3ef9b86448f951-Abstract.html"
}

@inproceedings{christiano2017deep,
  title = "Deep Reinforcement Learning from Human Preferences",
  author = "Christiano, Paul F. and Leike, Jan and Brown, Tom and Martic, Miljan and Legg, Shane and Amodei, Dario",
  booktitle = "Advances in Neural Information Processing Systems",
  volume = "30",
  pages = "4299--4307",
  year = "2017",
  url = "https://proceedings.neurips.cc/paper/2017/hash/d5e2c0adad503c91f91df240d0cd4e49-Abstract.html"
}

@article{bai2022training,
  title = "Training a Helpful and Harmless Assistant with Reinforcement Learning from Human Feedback",
  author = "Bai, Yuntao and Jones, Andy and Ndousse, Kamal and Askell, Amanda and Chen, Anna and DasSarma, Nova and Drain, Dawn and Fort, Stanislav and Ganguli, Deep and Henighan, Tom and Joseph, Nicholas and Kadavath, Saurav and Kernion, Jackson and Conerly, Tom and El-Showk, Sheer and Elhage, Nelson and Hatfield-Dodds, Zac and Hernandez, Danny and Hume, Tristan and Johnston, Scott and Kravec, Shauna and Lovitt, Liane and Nanda, Neel and Olsson, Catherine and Amodei, Dario and Brown, Tom and Clark, Jack and McCandlish, Sam and Olah, Chris and Mann, Ben and Kaplan, Jared",
  journal = "arXiv preprint arXiv:2204.05862",
  year = "2022",
  url = "https://arxiv.org/abs/2204.05862"
}

@inproceedings{singhal2023long,
  title = "A Long Way to Go: Investigating Length Correlations in {RLHF}",
  author = "Singhal, Prasann and Goyal, Tanya and Xu, Jiacheng and Durrett, Greg",
  booktitle = "First Conference on Language Modeling",
  year = "2024",
  url = "https://openreview.net/forum?id=G8LaO1P0xv"
}

@inproceedings{dubois2024alpacafarm,
  title = "{AlpacaFarm}: A Simulation Framework for Methods that Learn from Human Feedback",
  author = "Dubois, Yann and Li, Xuechen and Taori, Rohan and Zhang, Tianyi and Gulrajani, Ishaan and Ba, Jimmy and Guestrin, Carlos and Liang, Percy and Hashimoto, Tatsunori B.",
  booktitle = "Advances in Neural Information Processing Systems",
  volume = "36",
  pages = "30039--30069",
  year = "2023",
  url = "https://proceedings.neurips.cc/paper_files/paper/2023/hash/5fc47800ee5b30b8777fdd30abcaaf3b-Abstract-Conference.html"
}

@article{saito2023verbosity,
  title = "Verbosity Bias in Preference Labeling by Large Language Models",
  author = "Saito, Keita and Wachi, Akifumi and Wataoka, Koki and Akimoto, Youhei",
  journal = "arXiv preprint arXiv:2310.10076",
  year = "2023",
  url = "https://arxiv.org/abs/2310.10076"
}

@inproceedings{hu2022lora,
  title = "{LoRA}: Low-Rank Adaptation of Large Language Models",
  author = "Hu, Edward J. and Shen, Yelong and Wallis, Phillip and Allen-Zhu, Zeyuan and Li, Yuanzhi and Wang, Shean and Wang, Lu and Chen, Weizhu",
  booktitle = "Proceedings of the International Conference on Learning Representations",
  year = "2022",
  url = "https://openreview.net/forum?id=nZeVKeeFYf9"
}

@inproceedings{dettmers2024qlora,
  title = "{QLoRA}: Efficient Finetuning of Quantized {LLMs}",
  author = "Dettmers, Tim and Pagnoni, Artidoro and Holtzman, Ari and Zettlemoyer, Luke",
  booktitle = "Advances in Neural Information Processing Systems",
  volume = "36",
  pages = "10088--10115",
  year = "2023",
  url = "https://proceedings.neurips.cc/paper_files/paper/2023/hash/1feb87871436031bdc0f2beaa62a049b-Abstract-Conference.html"
}

@inproceedings{li2021prefix,
  title = "Prefix-Tuning: Optimizing Continuous Prompts for Generation",
  author = "Li, Xiang Lisa and Liang, Percy",
  booktitle = "Proceedings of the 59th Annual Meeting of the Association for Computational Linguistics and the 11th International Joint Conference on Natural Language Processing",
  pages = "4582--4597",
  year = "2021",
  url = "https://aclanthology.org/2021.acl-long.353/"
}

@inproceedings{lester2021power,
  title = "The Power of Scale for Parameter-Efficient Prompt Tuning",
  author = "Lester, Brian and Al-Rfou, Rami and Constant, Noah",
  booktitle = "Proceedings of the 2021 Conference on Empirical Methods in Natural Language Processing",
  pages = "3045--3059",
  year = "2021",
  url = "https://aclanthology.org/2021.emnlp-main.243/"
}

@inproceedings{houlsby2019parameter,
  title = "Parameter-Efficient Transfer Learning for {NLP}",
  author = "Houlsby, Neil and Giurgiu, Andrei and Jastrzebski, Stanislaw and Morrone, Bruna and De Laroussilhe, Quentin and Gesmundo, Andrea and Attariyan, Mona and Gelly, Sylvain",
  booktitle = "Proceedings of the 36th International Conference on Machine Learning",
  pages = "2790--2799",
  year = "2019",
  url = "https://proceedings.mlr.press/v97/houlsby19a.html"
}

@inproceedings{liu2024dora,
  title = "{DoRA}: Weight-Decomposed Low-Rank Adaptation",
  author = "Liu, Shih-Yang and Wang, Chien-Yi and Yin, Hongxu and Molchanov, Pavlo and Wang, Yu-Chiang Frank and Cheng, Kwang-Ting and Chen, Min-Hung",
  booktitle = "Proceedings of the 41st International Conference on Machine Learning",
  pages = "32100--32121",
  year = "2024",
  url = "https://proceedings.mlr.press/v235/liu24bn.html"
}

@inproceedings{zhang2023adaptive,
  title = "Adaptive Budget Allocation for Parameter-Efficient Fine-Tuning",
  author = "Zhang, Qingru and Chen, Minshuo and Bukharin, Alexander and He, Pengcheng and Cheng, Yu and Chen, Weizhu and Zhao, Tuo",
  booktitle = "Proceedings of the International Conference on Learning Representations",
  year = "2023",
  url = "https://openreview.net/forum?id=lq62uWRJjiY"
}

@inproceedings{aghajanyan2021intrinsic,
  title = "Intrinsic Dimensionality Explains the Effectiveness of Language Model Fine-Tuning",
  author = "Aghajanyan, Armen and Gupta, Sonal and Zettlemoyer, Luke",
  booktitle = "Proceedings of the 59th Annual Meeting of the Association for Computational Linguistics",
  pages = "7319--7328",
  year = "2021",
  url = "https://aclanthology.org/2021.acl-long.568/"
}

@inproceedings{sharma2023truth,
  title = "The Truth is in There: Improving Reasoning in Language Models with Layer-Selective Rank Reduction",
  author = "Sharma, Pratyusha and Ash, Jordan T. and Misra, Dipendra",
  booktitle = "Proceedings of the International Conference on Learning Representations",
  year = "2024",
  url = "https://openreview.net/forum?id=ozX92bu8VA"
}

@inproceedings{biderman2023pythia,
  title = "Pythia: A Suite for Analyzing Large Language Models Across Training and Scaling",
  author = "Biderman, Stella and Schoelkopf, Hailey and Anthony, Quentin Gregory and Bradley, Herbie and O'Brien, Kyle and Hallahan, Eric and Khan, Mohammad Aflah and Purohit, Shivanshu and Prashanth, USVSN Sai and Raff, Edward and Skowron, Aviya and Sutawika, Lintang and van der Wal, Oskar",
  booktitle = "Proceedings of the 40th International Conference on Machine Learning",
  pages = "2397--2430",
  year = "2023",
  url = "https://proceedings.mlr.press/v202/biderman23a.html"
}

@inproceedings{black2022gpt,
  title = "{GPT-NeoX-20B}: An Open-Source Autoregressive Language Model",
  author = "Black, Sid and Biderman, Stella and Hallahan, Eric and Anthony, Quentin and Gao, Leo and Golding, Laurence and He, Horace and Leahy, Connor and McDonell, Kyle and Phang, Jason and Pieler, Michael and Prashanth, USVSN Sai and Purohit, Shivanshu and Reynolds, Laria and Tow, Jonathan and Wang, Ben and Weinbach, Samuel",
  booktitle = "Proceedings of {B}ig{S}cience Episode {\#}5 -- Workshop on Challenges {\&} Perspectives in Creating Large Language Models",
  pages = "95--136",
  year = "2022",
  url = "https://aclanthology.org/2022.bigscience-1.9/"
}

@article{yang2024qwen25,
  title = "{Qwen2.5} Technical Report",
  author = "Yang, An and Yang, Baosong and Zhang, Beichen and Hui, Binyuan and Zheng, Bo and Yu, Bowen and Li, Chengyuan and Liu, Dayiheng and Huang, Fei and others",
  journal = "arXiv preprint arXiv:2412.15115",
  year = "2024",
  url = "https://arxiv.org/abs/2412.15115"
}

@article{shazeer2020glu,
  title = "{GLU} Variants Improve Transformer",
  author = "Shazeer, Noam",
  journal = "arXiv preprint arXiv:2002.05202",
  year = "2020",
  url = "https://arxiv.org/abs/2002.05202"
}

@inproceedings{zhang2019rmsnorm,
  title = "Root Mean Square Layer Normalization",
  author = "Zhang, Biao and Sennrich, Rico",
  booktitle = "Advances in Neural Information Processing Systems",
  volume = "32",
  pages = "12360--12371",
  year = "2019",
  url = "https://proceedings.neurips.cc/paper/2019/hash/1e8a19426224ca89e83cef47f1e7f53b-Abstract.html"
}

@article{mukherjee2023orca,
  title = "Orca: Progressive Learning from Complex Explanation Traces of {GPT-4}",
  author = "Mukherjee, Subhabrata and Mitra, Arindam and Jawahar, Ganesh and Agarwal, Sahaj and Palangi, Hamid and Awadallah, Ahmed",
  journal = "arXiv preprint arXiv:2306.02707",
  year = "2023",
  url = "https://arxiv.org/abs/2306.02707"
}

@inproceedings{leviathan2023fast,
  title = "Fast Inference from Transformers via Speculative Decoding",
  author = "Leviathan, Yaniv and Kalman, Matan and Matias, Yossi",
  booktitle = "Proceedings of the 40th International Conference on Machine Learning",
  pages = "19274--19286",
  year = "2023",
  url = "https://proceedings.mlr.press/v202/leviathan23a.html"
}

@inproceedings{pope2023efficiently,
  title = "Efficiently Scaling Transformer Inference",
  author = "Pope, Reiner and Douglas, Sholto and Chowdhery, Aakanksha and Devlin, Jacob and Bradbury, James and Heek, Jonathan and Xiao, Kefan and Agrawal, Shivani and Dean, Jeff",
  booktitle = "Proceedings of Machine Learning and Systems",
  volume = "5",
  pages = "606--624",
  year = "2023",
  url = "https://proceedings.mlsys.org/paper_files/paper/2023/hash/c4be71ab8d24cdfb45e3d06dbfca2780-Abstract-mlsys2023.html"
}

@inproceedings{dao2022flashattention,
  title = "{FlashAttention}: Fast and Memory-Efficient Exact Attention with {IO}-Awareness",
  author = "Dao, Tri and Fu, Daniel Y. and Ermon, Stefano and Rudra, Atri and R{\'e}, Christopher",
  booktitle = "Advances in Neural Information Processing Systems",
  volume = "35",
  pages = "16344--16359",
  year = "2022",
  url = "https://proceedings.neurips.cc/paper_files/paper/2022/hash/67d57c32e20fd0a7a302cb81d36e40d5-Abstract-Conference.html"
}

@inproceedings{kwon2023efficient,
  title = "Efficient Memory Management for Large Language Model Serving with {PagedAttention}",
  author = "Kwon, Woosuk and Li, Zhuohan and Zhuang, Siyuan and Sheng, Ying and Zheng, Lianmin and Yu, Cody Hao and Gonzalez, Joseph E. and Zhang, Hao and Stoica, Ion",
  booktitle = "Proceedings of the 29th Symposium on Operating Systems Principles",
  pages = "611--626",
  year = "2023",
  url = "https://dl.acm.org/doi/10.1145/3600006.3613165"
}

@article{zhou2023instruction,
  title = "Instruction-Following Evaluation for Large Language Models",
  author = "Zhou, Jeffrey and Lu, Tianjian and Mishra, Swaroop and Brahma, Siddhartha and Basu, Sujoy and Luan, Yi and Zhou, Denny and Hou, Le",
  journal = "arXiv preprint arXiv:2311.07911",
  year = "2023",
  url = "https://arxiv.org/abs/2311.07911"
}

\clearpage
\appendix
\setlength{\intextsep}{4pt plus 1pt minus 1pt}
\setlength{\textfloatsep}{4pt plus 1pt minus 1pt}

\section{Formal Proofs of Propositions 1 and 2}\label{sec:appendix_proof}
This appendix gives the formal proofs of Propositions~\ref{prop:adapter_count} and~\ref{prop:merge_equivalence}, together with remarks clarifying their scope.

\begin{proof}[Proof of Proposition~\ref{prop:adapter_count}]
Let $\mathcal{W}=\mathbb{R}^{d_1\times d_2}$ denote the space of matrices that can be updated by a targeted module.
Under full-parameter fine-tuning, write the update as $\Delta W\in\mathcal{W}$.
For the standard matrix basis $\{E_{ij}:1\leq i\leq d_1,\,1\leq j\leq d_2\}$, every update has the unique expansion
\[
\Delta W=\sum_{i=1}^{d_1}\sum_{j=1}^{d_2} \delta_{ij}E_{ij}.
\]
The coefficients $\delta_{ij}$ are independent trainable scalars. Therefore, the full-parameter representation exposes exactly $d_1d_2$ trainable entries.

For LoRA, define the factor parameter space
\[
\mathcal{P}=\mathbb{R}^{d_1\times r}\times\mathbb{R}^{r\times d_2}
\]
and the update map $\phi:\mathcal{P}\to\mathcal{W}$ by
\[
\phi(B,A)=\frac{\alpha}{r}BA.
\]
The first factor contains $d_1r$ scalar entries and the second contains $rd_2$ scalar entries.
Because $\alpha/r\ne 0$ is fixed, it introduces no trainable scalar and does not change the number of entries exposed to the optimizer. Therefore,
\[
\dim_{\mathrm{entries}}(\mathcal{P})=d_1r+rd_2=r(d_1+d_2).
\]
Moreover, for every $(B,A)\in\mathcal{P}$ and indices $i,j$,
\[
\phi(B,A)_{ij}=\frac{\alpha}{r}\sum_{k=1}^{r}B_{ik}A_{kj},
\]
so $\operatorname{rank}(\phi(B,A))\leq r$.
This establishes both the claimed factor-entry count and the rank constraint on the represented update.

The factorization map need not be injective: for every invertible $G\in\mathbb{R}^{r\times r}$,
\[
\phi(BG,G^{-1}A)=\frac{\alpha}{r}BGG^{-1}A=\phi(B,A).
\]
The identity shows that distinct factor pairs can represent the same matrix. At full-rank factors, this change-of-basis redundancy has $r^2$ dimensions. Accounting for that redundancy gives a local dimension of $r(d_1+d_2-r)$ for the represented rank-$r$ manifold.
The optimizer stores gradients and (when applicable) state tensors for the individual entries of $B$ and $A$. Consequently, the implementation-level trainable-parameter count remains $r(d_1+d_2)$.

Comparing this count with the $d_1d_2$ independent entries of full fine-tuning gives a strict reduction if and only if
\[
r(d_1+d_2)<d_1d_2.
\]
This completes the proof.
\end{proof}

\begin{remark}[Interpretation and scope]
The proposition is a parameter-count identity; it does not claim that LoRA projects updates onto semantic or verbosity-specific singular directions.
The usual initialization with $B=0$ makes the initial effective update zero, while later update directions are determined by optimization through the factorized parameterization.
Consequently, parameter reduction alone does not imply a particular change in generation length, preference accuracy, total memory, or training speed; those quantities are measured empirically in Section~\ref{sec:evaluation}.
\end{remark}

\begin{proof}[Proof of Proposition~\ref{prop:merge_equivalence}]
Let the common bias term, if present, be $b$.
The unmerged inference path computes
\[
z_{\mathrm{unmerged}} = W_0h + \frac{\alpha}{r}B(Ah) + b.
\]
By associativity and distributivity of matrix multiplication,
\[
\begin{aligned}
z_{\mathrm{unmerged}}
&= \left(W_0 + \frac{\alpha}{r}BA\right)h+b \\
&= W^*h+b
 = z_{\mathrm{merged}}.
\end{aligned}
\]
The equality holds for every admissible input vector $h$ and does not depend on the values or rank of $B$ and $A$, provided $r\ne0$ and the scalar $\alpha$ is fixed.
If the module is embedded in a deterministic downstream network whose parameters and subsequent operations are unchanged, equal module outputs induce equal downstream activations and logits by composition.
The condition that adapter dropout is disabled is required because a stochastic dropout mask inserted into the unmerged branch is not represented by the fixed merged matrix.
\end{proof}

\begin{remark}[Deployment scope]
The result establishes exact algebraic equivalence for a single merged linear module in inference mode.
It does not establish an end-to-end systems benefit: an unmerged multi-tenant deployment still incurs adapter-specific computation and bookkeeping, while numerical effects from quantization or finite-precision operation order may produce small implementation-level differences.
\end{remark}

\section{Experimental Reference Tables}\label{sec:appendix_reference_tables}
Table~\ref{tab:backbones} describes the two backbone architectures and their attention-adapter parameterizations.
Pythia uses fused query-key-value projections, whereas Qwen uses separate query, key, and value projections with grouped-query attention.
Consequently, the same attention-adaptation scope exposes different parameter counts: 7.86M for Pythia and 7.37M for Qwen.
The table relates these counts to the module layouts used in the cross-backbone comparison.

Table~\ref{tab:baselines} defines the starting checkpoint and update parameterization of each comparison branch.
For controlled DPO and DrDPO, both branches start from a shared full-parameter SFT checkpoint and retain the same native preference objective.
The full-parameter branch updates the backbone, whereas LOCUS updates its low-rank adapter.
The SamPO row identifies a separate continued-adaptation comparison against the official released checkpoint.

Table~\ref{tab:datasets} distinguishes the training pools, development splits, and final evaluation sets.
The controlled HH comparisons use development data for selection, with separate confirmation where specified, before evaluation on the 8,552 frozen test pairs.
The SamPO comparison uses its established 256-example evaluation split.
The additional Anthropic Harmless and Orca DPO rows identify the safety-refusal and instruction-following datasets. Their evaluation scope is limited to the corresponding 256-example development splits; no external frozen-test claim is made for these rows.

\begin{table*}[t]
\centering
\small
\setlength{\tabcolsep}{4pt}
\begin{tabular}{@{}p{0.15\textwidth}p{0.29\textwidth}p{0.32\textwidth}p{0.16\textwidth}@{}}
\toprule
\textbf{Backbone} & \textbf{Architecture} & \textbf{LOCUS targets} & \textbf{Trainable parameters} \\
\midrule
Pythia-2.8B & GPT-NeoX; 32 layers; hidden size 2560; parallel attention/MLP blocks. & Fused QKV and attention output modules. & 7.86M (0.2826\%). \\
Qwen2.5-3B & Modern Llama-style architecture; 36 layers; hidden size 2048; grouped-query attention; SwiGLU; RMSNorm. & Separate \texttt{q\_proj}, \texttt{k\_proj}, \texttt{v\_proj}, and \texttt{o\_proj}. & 7.37M (0.2383\%). \\
\bottomrule
\end{tabular}
\caption{Backbones and backbone-specific LOCUS parameterizations.}
\label{tab:backbones}

\vspace{8pt}
\begin{tabular}{@{}p{0.16\textwidth}p{0.23\textwidth}p{0.27\textwidth}p{0.25\textwidth}@{}}
\toprule
\textbf{Comparison} & \textbf{Starting point} & \textbf{Full-parameter branch} & \textbf{Low-rank LOCUS branch} \\
\midrule
Pythia DPO & Shared full-parameter SFT. & Full-parameter DPO with native objective. & Low-rank DPO with same objective and splits. \\
Pythia DrDPO & Shared full-parameter SFT. & Full-parameter DrDPO with native objective. & Low-rank DrDPO with same objective and splits. \\
Pythia SamPO & Official released SamPO checkpoint. & Not a same-SFT comparison. & Continued low-rank adaptation under native SamPO. \\
Qwen DPO/DrDPO & Shared full-parameter SFT. & Full-parameter DPO/DrDPO. & Low-rank LOCUS; selection and confirmation precede frozen test evaluation. \\
\bottomrule
\end{tabular}
\caption{Full-parameter and low-rank branches used in the primary comparisons.}
\label{tab:baselines}

\vspace{8pt}
\begin{tabular}{@{}p{0.18\textwidth}p{0.22\textwidth}p{0.28\textwidth}p{0.22\textwidth}@{}}
\toprule
\textbf{Dataset / protocol} & \textbf{Training pool} & \textbf{Selection / confirmation} & \textbf{Frozen evaluation} \\
\midrule
HH / Pythia DPO & 160,800-pair pool. & One 256-example development split. & 8,552 test pairs. \\
HH / Pythia DrDPO & 160,800-pair pool. & Disjoint 256-example selection and confirmation splits. & 8,552 test pairs. \\
HH / SamPO & Released checkpoint; no retraining baseline. & Established 256-example evaluation split. & 256 examples. \\
HH / Qwen validation & 160,288 pairs after two 256-example development splits. & 256-example selection and 256-example confirmation. & 8,552 test pairs. \\
Harmless dataset & 42,130 HH-RLHF \texttt{harmless-base} training pairs; safety refusal. & One 256-example development split. & Development evaluation only. \\
Orca DPO dataset & 12,112 context-filtered \texttt{Intel/orca\_dpo\_pairs} training pairs; instruction following. & One 256-example development split. & Development evaluation only. \\
\bottomrule
\end{tabular}
\caption{Datasets, development splits, and frozen evaluation sets.}
\label{tab:datasets}
\end{table*}

\section{Additional Experimental Results}\label{sec:appendix_results}
This appendix records the primary HH comparisons and the cross-task development evaluations reported in Section~\ref{sec:evaluation}.
Table~\ref{tab:main_results} consolidates the primary Anthropic HH comparisons across both backbone families and the three preference objectives.
For DPO and DrDPO, both branches start from the same SFT checkpoint; for SamPO, LOCUS adapts the released SamPO checkpoint. The table reports the measured token and internal preference-diagnostic changes without introducing an additional evaluation protocol.

\begin{table*}[t]
\centering
\resizebox{\textwidth}{!}{%
\small
\setlength{\tabcolsep}{8pt}
\begin{tabular}{@{}lcccc@{}}
\toprule
\textbf{Protocol and backbone} & \textbf{Tokens} & \textbf{Reduction} & \textbf{Preference accuracy} & \textbf{Accuracy change; trainable parameters} \\
\midrule
SamPO / Anthropic Helpful and Harmless (Pythia-2.8B) & 132.77$\to$\textbf{79.88} & \textbf{+39.84\%} & 53.52$\to$53.52\% & \hphantom{-}0.00 pp; 7.86M/0.28\% \\
DPO / Anthropic Helpful and Harmless (Pythia-2.8B) & 137.67$\to$\textbf{109.12} & \textbf{+20.73\%} & 48.40$\to$48.39\% & -0.01 pp; 7.86M/0.28\% \\
DrDPO / Anthropic Helpful and Harmless (Pythia-2.8B) & 145.61$\to$\textbf{108.79} & \textbf{+25.29\%} & 48.39$\to$48.26\% & -0.13 pp; 7.86M/0.28\% \\
\midrule
DPO / Anthropic Helpful and Harmless (Qwen2.5-3B) & 108.26$\to$\textbf{92.16} & \textbf{+14.87\%} & 48.69$\to$48.64\% & -0.05 pp; 7.37M/0.24\% \\
DrDPO / Anthropic Helpful and Harmless (Qwen2.5-3B) & 111.58$\to$\textbf{91.97} & \textbf{+17.58\%} & 48.64$\to$48.55\% & -0.09 pp; 7.37M/0.24\% \\
\bottomrule
\end{tabular}%
}
\caption{Primary results across both backbone families.}
\label{tab:main_results}
\end{table*}

Table~\ref{tab:crosstask} gives the exact values behind the cross-task comparison. The Anthropic Helpful and Harmless row reproduces the 8,552-pair frozen-test result from Table~\ref{tab:main_results}; the Harmless and Orca DPO rows use 256-example development evaluations. In every row, the preference values are internal chosen-vs-rejected diagnostics rather than external quality measurements.

\begin{table*}[t]
\centering
\small
\setlength{\tabcolsep}{4pt}
\begin{tabular}{@{}p{0.27\textwidth}p{0.08\textwidth}rrrrrr@{}}
\toprule
\textbf{Dataset} & \textbf{Split} & \shortstack{\textbf{Base}\\\textbf{tokens}} & \shortstack{\textbf{LOCUS}\\\textbf{tokens}} & \textbf{Reduction} & \shortstack{\textbf{Base}\\\textbf{accuracy}} & \shortstack{\textbf{LOCUS}\\\textbf{accuracy}} & \shortstack{\textbf{Accuracy}\\\textbf{change}} \\
\midrule
Anthropic Helpful and Harmless & 8,552 test & 137.67 & 109.12 & 20.73\% & 48.40\% & 48.39\% & -0.01 pp \\
Anthropic Harmless & 256 dev. & 137.53 & 102.75 & 25.29\% & 57.03\% & 58.20\% & +1.17 pp \\
Orca DPO & 256 dev. & 175.95 & 35.25 & 79.97\% & 69.14\% & 82.81\% & +13.67 pp \\
\bottomrule
\end{tabular}
\caption{Cross-task token and accuracy comparisons.}
\label{tab:crosstask}
\end{table*}

\end{document}